\documentclass{article}

\usepackage[utf8]{inputenc} %
\usepackage[T1]{fontenc}    %
\usepackage{hyperref}       %
\usepackage[round]{natbib}
\usepackage{url}            %
\usepackage{booktabs}       %
\usepackage{amsfonts}
\usepackage{amsmath} %
\usepackage{amsthm}
\usepackage{amssymb}
\usepackage{nicefrac}       %
\usepackage{microtype}      %
\usepackage[ruled,norelsize,noend]{algorithm2e}
\usepackage{bm}
\usepackage{graphicx}
\usepackage{color}
\usepackage{stmaryrd}
\usepackage{subfiles}
\usepackage{subcaption}
\usepackage{wrapfig}
\usepackage[parfill]{parskip}  
\usepackage{tikz}
\usetikzlibrary{arrows, positioning, calc}
\usetikzlibrary{decorations.pathreplacing, patterns}

\usepackage[nameinlink]{cleveref}

\makeatletter
\def\NAT@spacechar{~}%
\makeatother

\usepackage{etoolbox}
\newtoggle{supplementary}
\toggletrue{supplementary}

\definecolor{linkcolor}{RGB}{83,83,182}
\hypersetup{
    colorlinks=true,
    citecolor=linkcolor,
    linkcolor=linkcolor,
    urlcolor=linkcolor
}

\usepackage{array}
\newcolumntype{H}{>{\setbox0=\hbox\bgroup}m{0em}<{\egroup}@{}}

\usepackage{thmtools, thm-restate}
\newtheorem{theorem}{Theorem}[section]
\newtheorem{lemma}[theorem]{Lemma}
\newtheorem{proposition}[theorem]{Proposition}

\newtheorem{assumption}[theorem]{Assumption}
\newtheorem{definition}[theorem]{Definition}

\newcounter{condition} %
\newenvironment{condition}{
  \[
  \refstepcounter{condition}
}{\]}

\DeclareMathOperator*{\argmin}{arg\,min}
\DeclareMathOperator*{\sign}{sign}
\DeclareMathOperator{\Id}{Id}
\DeclareMathOperator*{\supp}{supp}
\DeclareMathOperator{\st}{ST}

\def\Rset{\mathbb R}
\def\Nset{\mathbb N}

\usepackage[textwidth=\marginparwidth, textsize=tiny]{todonotes}
\def\ie{\emph{i.e.}~}
\def\eg{\emph{e.g.}~}
\newcommand{\cmark}{$\checkmark$}%
\newcommand{\xmark}{$\times$}%

\graphicspath{{figures/}}

\renewcommand\labelenumi{(\roman{enumi})}
\renewcommand\theenumi\labelenumi

\title{Learning step sizes for unfolded sparse coding}

\author{
Pierre Ablin$^*$ \and Thomas Moreau$^*$ \and Mathurin Massias \and Alexandre Gramfort
\\
\\
Inria, Universit\'e Paris-Saclay
\\
Saclay, France
\\
$^*$: both authors contributed equally.
}

\begin{document}

\maketitle

\begin{abstract}

Sparse coding is typically solved by iterative optimization techniques, such as the Iterative Shrinkage-Thresholding Algorithm (ISTA).
Unfolding and learning weights of ISTA using neural networks is a practical way to accelerate estimation.
In this paper, we study the selection of adapted step sizes for ISTA.
We show that a simple step size strategy can improve the convergence rate of ISTA by leveraging the sparsity of the iterates.
However, it is impractical in most large-scale applications.
Therefore, we propose a network architecture where only the step sizes of ISTA are learned.
We demonstrate that for a large class of unfolded algorithms, if the algorithm converges to the solution of the Lasso, its last layers correspond to ISTA with learned step sizes.
Experiments show that our method is competitive with state-of-the-art networks when the solutions are sparse enough.
\end{abstract}

\section{Introduction}
\label{sec:intro}

The resolution of convex optimization problems by iterative algorithms has become a key part of machine learning and signal processing pipelines.
Amongst these problems, special attention has been devoted to the Lasso \citep{tibshirani1996regression}, due to the attractive sparsity properties of its solution (see \citealt{Hastie2015} for an extensive review).
For a given input $x \in \Rset^n~,$ a dictionary $D \in \Rset^{n \times m}$ and a regularization parameter $\lambda > 0 ~,$ the Lasso problem is
\begin{equation}\label{eq:lasso}
    z^*(x) \in \argmin_{z \in \Rset^m} F_x(z) \quad \textrm{with}\quad F_x(z) \triangleq \frac{1}{2} \| x - Dz \|^2 + \lambda \|z \|_1 \enspace.
\end{equation}

A variety of algorithms exist to solve Problem~\eqref{eq:lasso}, \eg proximal coordinate descent
\citep{Tseng01,Friedman2007}, Least Angle Regression \citep{Efron_Hastie_Johnstone_Tibshirani04} or proximal splitting methods \citep{Combettes2011}.
The focus of this paper is on the Iterative Shrinkage-Thresholding Algorithm (ISTA, \citealt{Daubechies2004}), which is a proximal-gradient method applied to Problem~\eqref{eq:lasso}.
ISTA starts from $z^{(0)} = 0$ and iterates
\begin{equation}
    \label{eq:ista}
    z^{(t+1)} = \st\left(z^{(t)} - \frac{1}{L}D^{\top}(Dz^{(t)} - x),
                         \frac{\lambda}{L}\right) \enspace,
\end{equation}
where $\st$ is the soft-thresholding operator defined as $\st(x, u) \triangleq \text{sign}(x)\max(|x| - u, 0)~,$ and $L$ is the greatest eigenvalue of $D^\top D~.$
In the general case, ISTA converges at rate $1/t~,$ which can be improved to the \emph{optimal} rate $1/t^2$ \citep{Nesterov83}.
However, this optimality stands in the worst possible case, and linear rates are achievable in practice \citep{liang2014local}.

A popular line of research to improve the speed of Lasso solvers is to try to identify the support of $z^*~,$ in order to diminish the size of the optimization problem \citep{ElGhaoui_Viallon_Rabbani12,Ndiaye_Fercoq_Gramfort_Salmon16b,Johnson_Guestrin15,Massias_Gramfort_Salmon18}.
Once the support is identified, larger steps can also be taken, leading to improved rates for first order algorithms \citep{liang2014local,Poon18,sun19a}.

However, these techniques only consider the case where a single Lasso problem is solved.
When one wants to solve the Lasso for many samples $\{x^i\}_{i =1}^N$ -- \eg in dictionary learning \citep{Olshausen1997} -- it is proposed by~\citet{Gregor2010} to \emph{learn} a $T$-layers neural network of parameters $\Theta~,$ $\Phi_{\Theta}: \Rset^n \rightarrow \Rset^m$ such that $\Phi_{\Theta}(x) \simeq z^*(x)~.$
This Learned-ISTA (LISTA) algorithm yields better solution estimates than ISTA on new samples for the same number of iterations/layers.
This idea has led to a profusion of literature (summarized in \autoref{tab:literature_network} in appendix).
Recently, it has been hinted by \citet{Zhang2018, Ito2018, Liu2019} that only a few well-chosen parameters can be learned while retaining the performances of LISTA.

In this article, we study strategies for LISTA where only step sizes are learned.
In \autoref{sec:steps}, we propose Oracle-ISTA, an analytic strategy to obtain larger step sizes in ISTA.
We show that the proposed algorithm's convergence rate can be much better than that of ISTA.
However, it requires computing a large number of Lipschitz constants which is a burden in high dimension.
This motivates the introduction of Step-LISTA (SLISTA) networks in \autoref{sec:learning}, where only a step size parameter is learned per layer.
As a theoretical justification, we show in \autoref{thm:coupling} that the last layers of \emph{any} deep LISTA network converging on the Lasso \emph{must} correspond to ISTA iterations with learned step sizes.
We validate the soundness of this approach with numerical experiments in \Cref{sec:expes}.

\section{Notation and Framework}

\paragraph{Notation}
The $\ell_2$ norm on $\Rset^n$ is $\Vert \cdot \Vert$.
For $p \in [1, \infty]~,$ $\| \cdot \|_p$ is the $\ell_p$ norm.
The Frobenius matrix norm is $\|M\|_F$.
The identity matrix of size $m$ is $\Id_m~.$
$\st$ is the soft-thresholding operator.
Iterations are denoted $z^{(t)}~.$
$\lambda > 0$ is the regularization parameter.
The Lasso cost function is $F_x~.$
$\psi_\alpha(z, x) $ is one iteration of ISTA with step $\alpha$: $\psi_\alpha(z, x)  =\st(z - \alpha D^{\top}(Dz -x), \alpha \lambda)~.$
$\phi_\theta(z, x) $ is one iteration of a LISTA  layer with parameters $\theta= (W, \alpha, \beta)$: $\phi_\theta(z, x) =\st(z - \alpha W^{\top}(Dz -x), \beta \lambda)~.$

The set of integers between 1 and $m$ is $\llbracket 1, m \rrbracket~.$
Given $z\in \Rset^m~,$ the support is $\supp(z) = \{j \in \llbracket 1, m \rrbracket : z_j \neq 0\} \subset \llbracket 1, m \rrbracket~.$
For $S\subset \llbracket0, m\rrbracket$, $D_S\in \Rset^{n \times m}$ is the matrix containing the columns of $D$ indexed by $S$.
We denote $L_S,$ the greatest eigenvalue of $D_S^{\top}D_S$.
The equicorrelation set is $E = \{j \in \llbracket 1, m \rrbracket : \vert D_j^\top(D z^* - x) \vert = \lambda \}$.
The equiregularization set is $\mathcal{B}_{\infty} = \{x\in\Rset^n : \|D^{\top}x\|_{\infty} = 1\} $.
Neural networks parameters are between brackets, \eg{} $\Theta = \{\alpha^{(t)}, \beta^{(t)} \}_{t=0}^{T-1}~.$
The sign function is $\sign(x) = 1$ if $x>0$, $-1$ if $x< 0$ and $0$ is $x=0~.$

\paragraph{Framework}
This paragraph recalls some properties of the Lasso. \autoref{lemma:kkt} gives the first-order optimality conditions for the Lasso.

\begin{lemma}[Optimality for the Lasso]\label{lemma:kkt} The Karush-Kuhn-Tucker (KKT) conditions read
    \begin{equation}
        z^* \in \argmin F_x \Leftrightarrow \forall j \in \llbracket 1 , m \rrbracket, D_j ^\top(x - Dz^*) \in \lambda \partial | z^*_j | =
        \begin{cases}\begin{aligned}
        &\{ \lambda \sign z^*_j \}, &\mathrm{ if \, } z^*_j \neq 0 \enspace, \\
        & [-\lambda, \lambda],  &\mathrm{ if \, } z^*_j = 0 \enspace.
        \end{aligned}\end{cases} \label{eq:KKT}
    \end{equation}
\end{lemma}

Defining $\lambda_{\max} \triangleq \Vert D^\top x \Vert_\infty~,$ it holds
$\argmin F_x = \{0\} \Leftrightarrow  \lambda \geq \lambda_{\max}~.$
For \emph{some} results in \autoref{sec:steps}, we will need the following assumption on the dictionary $D$:

\begin{assumption}[Uniqueness assumption]\label{assum:uniqueness}
$D$ is such that the solution of Problem~\eqref{eq:lasso} is unique for all $\lambda$ and $x$ i.e. $\argmin F_x = \{z^*\}~.$
\end{assumption}

\Cref{assum:uniqueness} may seem stringent since whenever $m > n~,$ $F_x$ is not strictly convex.
However, it was shown in \citet[Lemma 4]{Tibshirani13} -- with earlier results from \citealt{Rosset_Zhu_Hastie04} --
that if $D$ is sampled from a continuous distribution, \Cref{assum:uniqueness} holds for $D$ with probability one.

\begin{definition}[Equicorrelation set]
    The KKT conditions motivate the introduction of the \emph{equicorrelation set} $E \triangleq \{j \in \llbracket 1, m \rrbracket : \vert D_j^\top(D z^* - x) \vert = \lambda \}~,$ since $j \notin E \implies z^*_j = 0~,$ \ie $E$ contains the support of any solution $z^*~.$

    When \Cref{assum:uniqueness} holds, we have $E = \supp(z^*)$ \citep[Lemma 16]{Tibshirani13}.
\end{definition}

We consider samples $x$ in the \emph{equiregularization} set
\begin{equation}
\mathcal{B}_{\infty} \triangleq \{x\in\Rset^n : \|D^{\top}x\|_{\infty} = 1\} \enspace ,
\end{equation}
which is the set of $x$ such that $\lambda_{\text{max}}(x)=1~.$
Therefore, when $\lambda\ge1~,$ the solution is $z^*(x) = 0$ for all $x\in \mathcal{B}_{\infty}~,$ and when $\lambda < 1~,$ $z^*(x) \neq 0$ for all $x\in \mathcal{B}_{\infty}~.$
For this reason, we assume $0 < \lambda < 1$ in the following.

\section{Better step sizes for ISTA}
\label{sec:steps}
The Lasso objective is the sum of a $L$-smooth function, $\frac{1}{2} \|x - D \cdot\|^2~,$ and a function with an explicit proximal operator, $\lambda \| \cdot \|_1~.$
Proximal gradient descent for this problem, with the sequence of step sizes $(\alpha^{(t)})$ consists in iterating
\begin{equation}
    \label{eq:pgd}
    z^{(t+1)} = \st\left(z^{(t)} - \alpha^{(t)} D^{\top}(Dz^{(t)} - x),
     \lambda \alpha^{(t)} \right) \enspace.
\end{equation}
ISTA follows these iterations with a constant step size $\alpha^{(t)} = 1/L~.$
In the following, denote \mbox{$\psi_\alpha(z, x)\triangleq \st(z - \alpha D^{\top}(Dz^{(t)} - x), \alpha \lambda)$}.
One iteration of ISTA can be cast as a majorization-minimization step~\citep{Beck2009}.
Indeed, for all $z \in \Rset^m~,$
\begin{align}
    F_x(z)  & = \tfrac{1}{2}\|x-D z^{(t)}\|^{2} + (z - z^{(t)})^{\top}D^{\top}(Dz^{(t)} - x) + \tfrac{1}{2} \Vert D (z - z^{(t)}) \Vert^2 + \lambda \|z\|_1 \label{eq:majorization} \\
      & \leq \underbrace{\tfrac{1}{2}\|x-D z^{(t)}\|^{2} + (z - z^{(t)})^{\top}D^{\top}(Dz^{(t)} - x) + \tfrac{L}{2}\|z - z^{(t)}\|^2 + \lambda \|z\|_1}_{\displaystyle \triangleq Q_{x, L}(z, z^{(t)})} \enspace, \label{eq:majorization2}
\end{align}
where we have used the inequality
$(z - z^{(t)})^{\top} D^\top D(z - z^{(t)}) \leq L\|z - z^{(t)}\|^2~.$
The minimizer of $Q_{x, L}(\cdot, z^{(t)})$ is $\psi_{1 / L}(z^{(t)}, x)$, which is the next ISTA step.

\paragraph{Oracle-ISTA: an accelerated ISTA with larger step sizes}
Since the iterates are sparse, this approach can be refined.
For $S\subset\llbracket 1, m \rrbracket~,$ let us define the $S$-smoothness of $D$ as
\begin{equation}
    \label{eq:spca}
    L_S \triangleq \max_z z^{\top}D^\top D z, \enspace \text{s.t.}  \enspace \|z\|=1 \text{ and supp}(z) \subset S \enspace,
\end{equation}
with the convention $L_\emptyset = L~.$
Note that $L_S$ is the greatest eigenvalue of $D_S^\top D_S$ where $D_S \in \Rset^{n \times |S|}$ is the columns of $D$ indexed by $S~.$
For all $S~,$ $L_S \leq L~,$ since $L$ is the solution of~\Cref{eq:spca} without support constraint.
Assume $\supp(z^{(t)})\subset S~.$
Combining \Cref{eq:majorization,eq:spca}, we have
\begin{equation}
    \label{eq:majspca}
    \forall z \enspace \text{s.t.} \enspace \supp(z) \subset S, \enspace F_x(z) \leq  Q_{x, L_S}(z, z^{(t)}) \enspace.
\end{equation}

The minimizer of the r.h.s is $z = \psi_{{1}/{L_S}}(z^{(t)}, x)~.$
Furthermore, the r.h.s. is a tighter upper bound than the one given in~\Cref{eq:majorization2} (see illustration in~\Cref{fig:surrogate}).
Therefore, using $z^{(t+1)} = \psi_{{1}/{L_S}}(z^{(t)}, x)$ minimizes a tighter upper bound, provided that the following condition holds
\begin{condition}
    \label{cond:star}\tag{$\star$}
    \supp(z^{(t+1)})\subset S~.
\end{condition}%

\begin{figure}[ht!]
\begin{minipage}{.38\textwidth}
    \includegraphics[width=\textwidth]{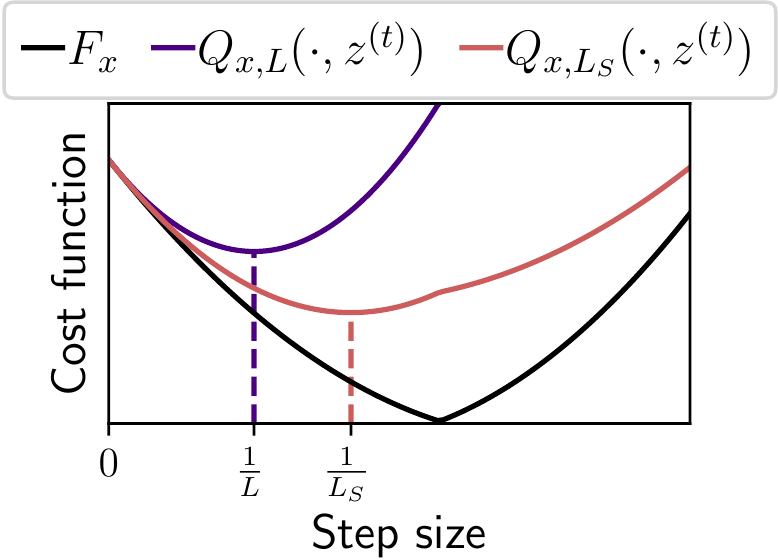}
\end{minipage}%
\begin{minipage}{.55\textwidth}
    \caption{Majorization illustration. If $z^{(t)}$ has support $S~,$ $Q_{x, L_S}(\cdot, z^{(t)})$ is a tighter upper bound of $F_x$ than $Q_{x, L}(\cdot, z^{(t)})$ on the set of points of support $S~.$}
    \label{fig:surrogate}
\end{minipage}
\end{figure}%
\begin{algorithm}[t]
\SetAlgoLined
\KwIn{ Dictionary $D~,$ target $x~,$ number of iterations $T$}
$z^{(0)} = 0$ \\
\For{$t=0, \ldots, T-1$}{
Compute $S = \supp(z^{(t)})$ and  $L_S$ using an oracle  ;\\
Set $y^{(t+1)} = \psi_{{1}/{L_S}}(z^{(t)}, x)$ ;\\
\lIf{\autoref{cond:star} : $\supp(y^{(t+1)}) \subset S$ \label{a}}{
Set $z^{(t+1)} = y^{(t+1)}$
}
{
\lElse{
Set $z^{(t+1)} = \psi_{{1}/{L}}(z^{(t)}, x)$
}
}
}
\KwOut{Sparse code $z^{(T)}$}
 \caption{Oracle-ISTA (OISTA) with larger step sizes \label{algo:oista}}
\end{algorithm}
Oracle-ISTA (OISTA) is an accelerated version of ISTA which leverages the sparsity of the iterates in order to use larger step sizes.
The method is summarized in \Cref{algo:oista}.
OISTA computes $y^{(t+1)} = \psi_{1/L_s}(z^{(t)}, x)~,$ using the larger step size $1 / L_S~,$ and checks if it satisfies the support \autoref{cond:star}.
When the condition is satisfied, the step can be safely accepted. In particular~\Cref{eq:majspca} yields $F_x(y^{(t+1)}) \leq F_x(z^{(t)})~.$
Otherwise, the algorithm falls back to the regular ISTA iteration with the smaller step size.
Hence, each iteration of the algorithm is guaranteed to decrease $F_x~.$
The following proposition shows that OISTA converges in iterates, achieves finite support identification, and eventually reaches a safe regime where \autoref{cond:star} is always true.

\begin{restatable}[Convergence, finite-time support identification and safe regime]{proposition}{cvgfinite}
\label{prop:cvgfinite}
When \Cref{assum:uniqueness} holds, the sequence $(z^{(t)})$ generated by the algorithm converges to $z^* =\argmin F_x~.$

Further, there exists an iteration $T^*$ such that for $t\geq T^*~,$ $\supp(z^{(t)}) = \supp(z^*) \triangleq S^*$ and \autoref{cond:star} is always statisfied.
\end{restatable}

\begin{proof}[Sketch of proof (full proof in \autoref{sub:proof:cvgfinite})]
Using Zangwill's global convergence theorem~\citep{Zangwill69}, we show that all accumulation points of $(z^{(t)})$ are solutions of Lasso. Since the solution is assumed unique, $(z^{(t)})$ converges to $z^*~.$
Then, we show that the algorithm achieves finite-support identification with a technique inspired by~\citet{Hale_Yin_Zhang2008}. The algorithm gets arbitrary close to $z^*~,$ eventually with the same support.
We finally show that in a neighborhood of $z^*~,$ the set of points of support $S^*$ is stable by $\psi_{1/L_S}(\cdot, x)~.$
The algorithm eventually reaches this region, and then \autoref{cond:star} is true.
\end{proof}
It follows that the algorithm enjoys the usual ISTA convergence results replacing $L$ with $L_{S^*}~.$
\begin{restatable}[Rates of convergence]{proposition}{sublinear}
\label{prop:rates}
For $t > T^*~,$ $F_x(z^{(t)}) - F_x(z^*) \leq L_{S^*} \frac{\|z^{*} - z^{(T^*)}\|^2}{2(t - T^*)}~.$\\
If additionally $\inf_{\|z\|=1} \|D_{S^*}z\|^2 = \mu^* > 0~,$ then the convergence rate for $t \geq T^*$ is\\
$F_x(z^{(t)}) - F_x(z^*) \leq (1 - \tfrac{\mu^*}{L_{S^*}})^{t - T^*}(F_x(z^{(T^*)}) - F_x(z^*))~.$
\end{restatable}
\begin{proof}[Sketch of proof (full proof in \autoref{sub:proof:rates})]
After iteration $T^*~,$ OISTA is equivalent to ISTA applied on $F_x(z)$ restricted to $z \in S^*~.$ This function is $L_{S^*}$-smooth, and $\mu^*$-strongly convex if $\mu^*>0~.$ Therefore, the classical ISTA rates apply with improved condition number.
\end{proof}

These two rates are tighter than the usual ISTA rates -- in the convex case $ L \frac{\|z^{*}\|^2}{2t}$ and in the $\mu$-strongly convex case $(1 - \frac{\mu^*}{L})^{t}(F_x(0) - F_x(z^*))$~\citep{Beck2009}.
Finally, the same way ISTA converges in one iteration when $D$ is orthogonal ($D^{\top}D=\Id_m$), OISTA converges in one iteration if $S^*$ is identified and $D_{S^*}$ is orthogonal.
\begin{proposition}
    \label{prop:orthogonal}
    Assume $D_{S^*}^{\top}D_{S^*} =L_{S^*} \Id_{|S^*|}~.$
    Then, $z^{(T^*+1)} = z^*~.$
\end{proposition}
\begin{proof}
    For $z$ \emph{s.t.} $\supp(z) = S^*~,$ $F_x(z) = Q_{x, L_S}(z, z^{(T^*)})~.$ Hence, the OISTA step minimizes $F_x ~.$
\end{proof}

\paragraph{Quantification of the rates improvement in a Gaussian setting}

The following proposition gives an asymptotic value for $\frac{L_S}{L}$ in a simple setting.
\begin{proposition}
    \label{prop:eig}
    Assume that the entries of $D \in \Rset^{n \times m}$ are i.i.d centered Gaussian variables with variance $1~.$
    Assume that $S$ consists of $k$ integers chosen uniformly at random in $\llbracket 1, m \rrbracket~.$
    Assume that $k, m, n \rightarrow +\infty$ with linear ratios $m/n \rightarrow \gamma, \enspace k /m \rightarrow \zeta \enspace .$ Then
    \begin{equation}
        \label{eq:ratioL}
        \frac{L_S}{L} \rightarrow \left(\frac{1 + \sqrt{\zeta\gamma}}{1 + \sqrt{\gamma}} \right)^2 \enspace.
    \end{equation}
\end{proposition}

This is a direct application of the Marchenko-Pastur law~\citep{marvcenko1967distribution}.
The law is illustrated on a toy dataset in \autoref{fig:lip_distrib}.
In \Cref{prop:eig}, $\gamma$ is the ratio between the number of atoms and number of dimensions, and the average size of $S$ is described by $\zeta \leq 1~.$
In an overcomplete setting, we have $\gamma \gg 1~,$ yielding the approximation of \Cref{eq:ratioL}: $L_S \simeq \zeta L~.$
Therefore, if $z^*$ is very sparse ($\zeta \ll 1$), the convergence rates of~\Cref{prop:rates} are much better than those of ISTA.

\paragraph{Example}
\Cref{fig:oista_ista} compares the OISTA, ISTA, and FISTA on a toy problem.
The improved rate of convergence of OISTA is illustrated.
Further comparisons are displayed in \autoref{fig:comparison_oista} for different regularization parameters $\lambda~.$
While this demonstrates a much faster rate of convergence, it requires computing several Lipschitz constants $L_S~,$ which is cumbersome in high dimension.
This motivates the next section, where we propose to~\emph{learn} those steps.

\begin{figure}[h]
\centering
\begin{minipage}{.48\textwidth}
    \includegraphics[width=\textwidth]{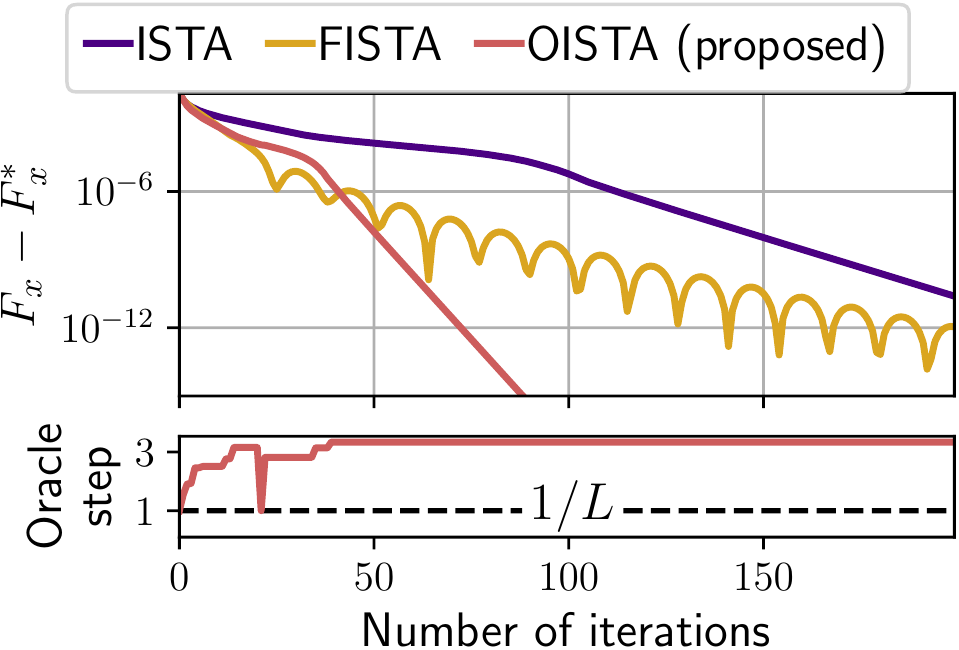}
\end{minipage}%
\hfill
\begin{minipage}{.45\textwidth}
    \caption{Convergence curves of OISTA, ISTA, and FISTA on a toy problem with $n=10~,$ $m=50~,$ $\lambda = 0.5~.$ The bottom figure displays the (normalized) steps taken by OISTA at each iteration. Full experimental setup described in \autoref{sec:supp_expe}.}
    \label{fig:oista_ista}
\end{minipage}
\begin{minipage}{.04\textwidth}
\
\end{minipage}
\end{figure}

\section{Learning unfolded algorithms}
\label{sec:learning}

\paragraph{Network architectures}

At each step, ISTA performs a linear operation to compute an update in the direction of the gradient $D^\top(Dz^{(t)} - x)$  and then an element-wise non linearity with the soft-thresholding operator $\st.$
The whole algorithm can be summarized as a recurrent neural network (RNN), presented in \autoref{fig:network_ista}.
\citet{Gregor2010} introduced Learned-ISTA (LISTA), a neural network constructed by unfolding this RNN $T$ times and learning the weights associated to each layer.
The unfolded network, presented in \autoref{fig:lista}, iterates $z^{(t+1)}=\st(W_x^{(t)}x + W_z^{(t)}z^{(t)}, \lambda\beta^{(t)})~.$
It outputs exactly the same vector as $T$ iterations of ISTA when $W_x^{(t)} = \frac{D^\top}{L}~,$ $W_z^{(t)} = \Id_m - \frac{D^\top D}{L}$ and $\beta^{(t)} = \frac{1}{L}~.$
Empirically, this network is able to output a better estimate of the sparse code solution with fewer operations.

\begin{figure}[hbtp!]
    \begin{subfigure}[b]{.4\textwidth}
        \centering
        \scalebox{1}{\begin{tikzpicture}[scale=1]

\tikzset{
    >=stealth',
    varstyle/.style={
           rectangle,
           fill=white,
           rounded corners,
           draw=black, very thick,
           text width=2em,
           minimum height=2em,
           text centered},
    op/.style={
           circle,
           draw=black, very thick,
           fill=white,
           text width=1em,
           minimum height=1em,
           text centered},
    st/.style={
           draw,
           thick,
           rectangle, 
           fill=white,
           text width=2em, 
           minimum height=2.5em,},
    pil/.style={
           ->,
           thick,
    }
};

\node[varstyle] (linear) {$W_x$};
\node[op, inner sep=4pt, right=1em of linear] (add) {};
\node[left=1em of linear] (X) {$x$} edge[pil] (linear.west);
\path (linear.east) edge[pil] (add.west);


\draw ($ (add.north) - (0, .15)$) -- ($ (add.south) + (0, .15)$);
\draw ($ (add.east) - (.15, 0)$) -- ($ (add.west) + (.15, 0)$);

\node[st, right=1em of add] (h) {};
\node[above=1em of h.center] (up) {};
\node[below=1em of h.center] (down) {};
\node[right=1em of h.center] (right) {};
\node[left=1em of h.center] (left) {};
\draw (h.center) -- (up) -- (down) -- (h.center) -- (right) -- (left);
\draw ($ (h.center) - (.2em,0em)  $) -- ($ (h.center) - (.8em,.7em)  $);
\draw ($ (h.center) + (.2em,0em)  $) -- ($ (h.center) + (.8em,.7em)  $);
\node[right=2em of h] (Z) {$z^*$};
\path (add.east)
edge[pil] (h.west);
\path (h.east)
edge[pil] (Z);
\node[inner sep=0,minimum size=0,right=1em of h] (branch) {};
\node[varstyle,below=.5em of h] (S) {$W_z$};
\draw[pil] (branch) |- (S.east);
\draw[pil] (S.west) -| (add);
\end{tikzpicture}}
        \caption{\textbf{ISTA} - Recurrent Neural Network}
        \label{fig:network_ista}
    \end{subfigure}
    \begin{subfigure}[b]{.55\textwidth}
        \scalebox{.8}{\begin{tikzpicture}[scale=1]

\tikzset{
    >=stealth',
    varstyle/.style={
           rectangle,
           fill=white,
           rounded corners,
           draw=black, very thick,
           text width=2em,
           minimum height=2em,
           text centered},
    op/.style={
           circle,
           draw=black, very thick,
           fill=white,
           text width=.5em,
           minimum height=.3em,
           text centered},
    st/.style={
           draw,
           thick,
           rectangle, 
           fill=white,
           text width=1.5em, 
           minimum height=2em,},
    pil/.style={
           ->,
           thick,
    }
};

\def\nlayer{2}

\node (X) {$x$};
\node[right=1em of X] (linear) {}; 
\node[below=3.7em of linear] (p) {};
\node[varstyle, above=1.45em of p] (B) {$W_x^{(0)}$};

\foreach \i in {1,...,\nlayer}{
    \node[st, right=.9em of p] (sta\i) {};
    \path (p.east) edge[pil] (sta\i.west);

    \node[varstyle, right=1em of sta\i] (S) {$W_z^{(\i)}$};
    \node[op, inner sep=4pt, right=1em of S] (p) {};
    \node[varstyle, above=.8em of p] (B\i) {$W_x^{(\i)}$};
    
    \draw ($ (p.north) - (0, .15)$) -- ($ (p.south) + (0, .15)$);
    \draw ($ (p.east) - (.15, 0)$) -- ($ (p.west) + (.15, 0)$);
    
    \draw ($ (sta\i.center) - (0,.9em)  $) -- ($ (sta\i.center) + (0,.9em)  $);
    \draw ($ (sta\i.center) - (.9em,0em)  $) -- ($ (sta\i.center) + (.9em,0em)  $);
    \draw ($ (sta\i.center) - (.2em,0em)  $) -- ($ (sta\i.center) - (.8em,.7em)  $);
    \draw ($ (sta\i.center) + (.2em,0em)  $) -- ($ (sta\i.center) + (.8em,.7em)  $);
    \path (sta\i.east) edge[pil] (S.west);
    \path (S.east) edge[pil] (p.west);
    \draw[pil] (X) -| (B\i) -- (p.north);
}
\draw[pil] (linear.center) -- (B) |- (sta1.west);

\node[st, right=1em of p] (sta) {};
    \draw ($ (sta.center) - (0,.9em)  $) -- ($ (sta.center) + (0,.9em)  $);
    \draw ($ (sta.center) - (.9em,0em)  $) -- ($ (sta.center) + (.9em,0em)  $);
    \draw ($ (sta.center) - (.2em,0em)  $) -- ($ (sta.center) - (.8em,.7em)  $);
    \draw ($ (sta.center) + (.2em,0em)  $) -- ($ (sta.center) + (.8em,.7em)  $);
    
\pgfmathtruncatemacro{\nlayer}{\nlayer + 1}
\node[right=1em of sta] (Z) {$z^{(\nlayer)}$};
\path (p.east) edge[pil] (sta.west);
\path (sta.east) edge[pil] (Z.west);

\end{tikzpicture}}
        \caption{\textbf{LISTA} - Unfolded network with $T=3$}
        \label{fig:lista}
    \end{subfigure}
    \caption{
        Network architecture for ISTA  (\emph{left}) and LISTA (\emph{right}).}
\end{figure}
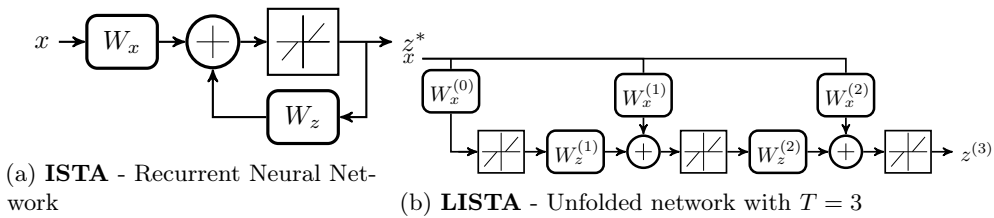

Due to the expression of the gradient, \citet{Chen2018} proposed to consider only a subclass of the previous networks, where the weights $W_x$ and $W_z$ are coupled via $W_z = \Id_m - W_x^\top D~.$
This is the architecture we consider in the following.
A layer of LISTA is a function $\phi_{\theta}: \Rset^m \times \Rset^n \to \Rset^m$ parametrized by $\theta = (W, \alpha, \beta) \in \Rset^{n \times m} \times \Rset^+_* \times \Rset^+_*$ such that
\begin{equation}
    \label{eq:lista_layer}
    \phi_{\theta}(z, x) = \st(z - \alpha W^{\top}(Dz -x), \beta\lambda) \enspace.
\end{equation}
Given a set of $T$ layer parameters $\Theta^{(T)} = \{\theta^{(t)}\}_{t=0}^{T-1}~,$ the LISTA network $\Phi_{\Theta^{(T)}}: \Rset^n \to \Rset^m$ is $\Phi_{\Theta^{(T)}}(x) = z^{(T)}(x)$ where $z^{(t)}(x)$ is defined by recursion
\begin{equation}
    z^{(0)}(x) = 0, ~~\text{and } ~~~
    z^{(t+1)}(x) = \phi_{\theta^{(t)}}(z^{(t)}(x), x)
    ~~\text{ for }~t\in\llbracket0, T-1\rrbracket\enspace .
\end{equation}
Taking $W=D~,$ $\alpha = \beta = \frac{1}{L}$ yields the same outputs as $T$ iterations of ISTA.

To alleviate the need to learn the large matrices $W^{(t)},$ \citet{Liu2019} proposed to use a shared analytic matrix $W_{\text{ALISTA}}$ for all layers.
The matrix is computed in a preprocessing stage by
\begin{equation}
    \label{eq:alista}
    W_{\text{ALISTA}} = \argmin_W \|W^\top D\|_F^2~~~~ s.t.~~~~ \text{diag}(W^\top D) = {\pmb 1}_m \enspace .
\end{equation}
Then, only the parameters $(\alpha^{(t)}, \beta^{(t)})$ are learned.
This effectively reduces the number of parameters from $(n m + 2)\times T$ to $2 \times T~.$
However, we will see that ALISTA fails in our setup.

\paragraph{Step-LISTA}
With regards to the study on step sizes for ISTA in \autoref{sec:steps}, we propose to \emph{learn} approximation of ISTA step sizes for the input distribution using the LISTA framework.
The resulting network, dubbed Step-LISTA (SLISTA), has $T$ parameters $\Theta_{\text{SLISTA}}=\{\alpha^{(t)}\}_{t=0}^{T-1}~,$ and follows the iterations:
\begin{equation}
    \label{eq:steprec}
    z^{(t+1)}(x) = \st(z^{(t)}(x) - \alpha^{(t)}D^{\top}(Dz^{(t)}(x) -x), \alpha^{(t)}\lambda) \enspace .
\end{equation}
This is equivalent to a coupling in the LISTA parameters: a LISTA layer $\theta = (W, \alpha, \beta)$ corresponds to a SLISTA layer if and only if $ \frac{\alpha}{\beta}W = D$.
This network aims at learning good step sizes, like the ones used in OISTA, without the computational burden of computing Lipschitz constants.
The number of parameters compared to the classical LISTA architecture $\Theta_{\text{LISTA}}$ is greatly diminished, making the network easier to train. Learning curves are shown in \autoref{fig:learning_curve} in appendix.
\autoref{fig:learned_steps} displays the learned steps of a SLISTA network on a toy example. The network learns larger step-sizes as the $1/L_S$'s increase.

\begin{figure}[hbtp!]
    \centering
    \begin{minipage}{.38\textwidth}
        \includegraphics[width=\textwidth]{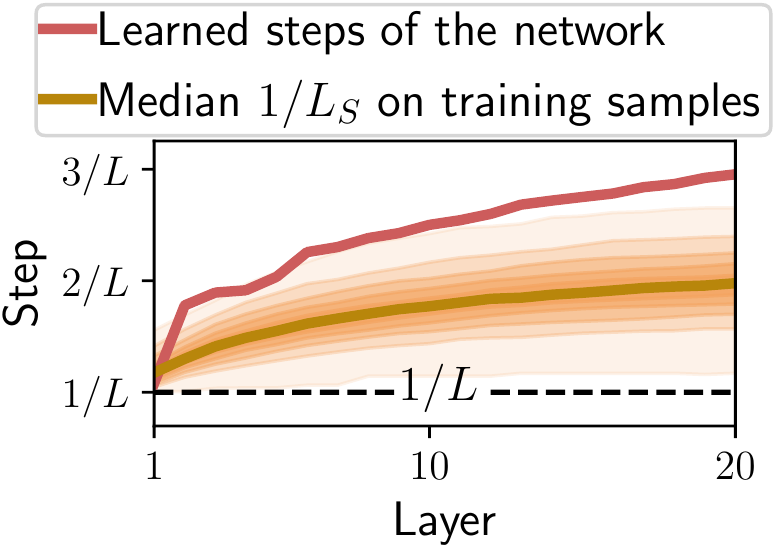}
    \end{minipage}%
    \hfill
    \begin{minipage}{.57\textwidth}
        \caption{Steps learned with a $20$ layers SLISTA network on a $10 \times 20$ problem. For each layer $t$ and each training sample $x$, we compute the support $S(x, t)$ of $z^{(t)}(x)$. The brown curves display the quantiles of the distribution of $1/L_{S(x, t)}$ for each layer $t~$. Full experimental setup described in \autoref{sec:supp_expe}.}
        \label{fig:learned_steps}
    \end{minipage}
\end{figure}

\paragraph{Training the network}
We consider the framework where the network learns to solve the Lasso on $\mathcal{B}_{\infty}$ in an \emph{unsupervised} way.
Given a distribution $p$ on $\mathcal{B}_{\infty}~,$ the network is trained by solving
\begin{equation}
    \label{eq:unsupervised}
    \tilde{\Theta}^{(T)} \in \arg\min_{\Theta^{(T)}} \mathcal{L}(\Theta^{(T)}) \triangleq \mathbb{E}_{x\sim p}[F_x(\Phi_{\Theta^{(T)}}(x))] \enspace .
\end{equation}
Most of the literature on learned optimization train the network with a different \emph{supervised} objective \citep{Gregor2010, Xin2016, Chen2018, Liu2019}. Given a set of pairs $(x^{i}, z^{i})~,$ the supervised approach tries to learn the parameters of the network such that $\Phi_\Theta(x^i) \simeq z^i$ \eg{} by minimizing $\|\Phi_\Theta(x^i) - z^i\|^2~.$
This training procedure differs critically from ours. For instance, ISTA does not converge for the supervised problem in general while it does for the unsupervised one. As \autoref{prop:pointwise_conv} shows, the unsupervised approach allows to \emph{learn to minimize} the Lasso cost function $F_x~.$
\begin{proposition}[Pointwise convergence]
    \label{prop:pointwise_conv}
    Let $\tilde{\Theta}^{(T)}$ found by solving Problem~\eqref{eq:unsupervised}.\\
    For $x\in \mathcal{B}_{\infty}$ such that $p(x) > 0~,$ $F_x(\Phi_{\tilde{\Theta}^{(T)}}(x)) \xrightarrow[T\to+\infty]{} F_x^*$ almost everywhere.
\end{proposition}
\begin{proof}
Let $\Theta_{\text{ISTA}}^{(T)}$ the parameters corresponding to ISTA \ie{} $\theta_{\text{ISTA}}^{(t)} = (D, 1/L, 1/L)~.$ For all $T~,$ we have $\mathbb{E}_{x\sim p}[F_x^*]\leq \mathbb{E}_{x\sim p}[F_x(\Phi_{\tilde{\Theta}^{(T)}}(x))]\leq \mathbb{E}_{x\sim p}[F_x(\Phi_{\Theta_{\text{ISTA}}^{(T)}}(x))]~.$ Since ISTA converges uniformly on any compact, the right hand term goes to $\mathbb{E}_{x\sim p}[F_x^*]~.$ Therefore, by the squeeze theorem, $ \mathbb{E}_{x\sim p}[F_x(\Phi_{\tilde{\Theta}^{(T)}}(x)) - F_x^*]\to 0~.$
This implies almost sure convergence of $F_x(\Phi_{\tilde{\Theta}^{(T)}}(x)) - F_x^*$ to $0$ since it is non-negative.
\end{proof}

\paragraph{Asymptotical weight coupling theorem}

In this paragraph, we show the main result of this paper: any LISTA network minimizing $F_x$ on $\mathcal B_\infty$ reduces to SLISTA in its deep layers (\autoref{thm:coupling}). It relies on the following Lemmas.

\begin{restatable}[Stability of solutions around $D_j$]{lemma}{stability}\label{lemma:stability}
    Let $D\in \Rset^{n\times m}$ be a dictionary with non-duplicated unit-normed columns.
    Let $c \triangleq \max_{l \neq j} |D_l^\top D_j| < 1~.$
    Then for all $j \in \llbracket 1, m \rrbracket$ and $\varepsilon \in \Rset^{m}$ such that $\|\varepsilon\| < \lambda(1-c)$ and $ D_j^\top\varepsilon=0~,$ the vector $(1  - \lambda ) e_j$ minimizes $F_x$ for $x = D_j + \varepsilon~.$
\end{restatable}

It can be proven by verifying the KKT conditions \eqref{eq:KKT} for $(1-\lambda)e_j~,$ detailed in \autoref{sec:proof:stability}.

\begin{restatable}[Weight coupling]{lemma}{wcoupling}
    \label{lemma:wcoupling}
    Let $D\in \Rset^{n\times m}$ be a dictionary with non-duplicated unit-normed columns.
    Let $\theta=(W, \alpha, \beta)$ a set of parameters.
    Assume that all the couples $(z^*(x), x) \in \Rset^m \times \mathcal{B}_{\infty}$ such that $z^*(x)\in \argmin F_x(z)$ verify $\phi_{\theta}(z^*(x), x) = z^*(x)$. Then, $\frac{\alpha}{\beta} W = D~.$
\end{restatable}

\begin{proof}[Sketch of proof (full proof in \autoref{sec:proof:wcoupling})]
    For $j \in \llbracket 1, m \rrbracket~,$ consider $x = D_j + \varepsilon~,$ with $\varepsilon^{\top}D_j=0~.$
    For $\|\varepsilon\|$ small enough, $x \in \mathcal{B}_{\infty}$ and $\varepsilon$ verifies the hypothesis of \autoref{lemma:stability}, therefore \mbox{$z^* = (1 - \lambda)e_j \in \argmin F_x~.$}
    Writing $\phi_{\theta}(z^*, x) = z^*$ for the $j$-th coordinate yields $\alpha W_j^{\top}(\lambda D_j + \varepsilon)= \lambda\beta~.$
    We can then verify that $(\alpha W_j^{\top} - \beta D_j^{\top})(\lambda D_j + \varepsilon) = 0~.$
    This stands for any $\varepsilon$ orthogonal to $D_j$ and of norm small enough.
    Simple linear algebra shows that this implies $\alpha W_j - \beta D_j = 0~.$
\end{proof}
\autoref{lemma:wcoupling} states that the Lasso solutions are fixed points of a LISTA layer only if this layer corresponds to a step size for ISTA.
The following theorem extends the lemma by continuity, and shows that the deep layers of any converging LISTA network must tend toward a SLISTA layer.

\begin{restatable}{theorem}{coupling}
    \label{thm:coupling}
    Let $D\in \Rset^{n\times m}$ be a dictionary with non-duplicated unit-normed columns.
    Let $\Theta^{(T)} = \{\theta^{(t)}\}_{t=0}^T$ be the parameters of a sequence of LISTA networks such that the transfer function of the layer $t$ is $z^{(t+1)} = \phi_{\theta^{(t)}}(z^{(t)}, x)~.$
    Assume that
    \begin{enumerate}
        \item\label{hyp:theta_cvg} the sequence of parameters converges \ie{}
            $\theta^{(t)} \xrightarrow[t\to\infty]{} \theta^* = (W^*, \alpha^*, \beta^*) \enspace ,$
        \item\label{hyp:out_cvg} the output of the network converges toward a solution $z^*(x)$ of the Lasso \eqref{eq:lasso} uniformly over the equiregularization set $\mathcal B_\infty~,$ \ie{} $\sup_{x \in \mathcal B_\infty}\|\Phi_{\Theta^{(T)}}(x) - z^*(x)\| \xrightarrow[T \to \infty]{} 0$ \enspace.
    \end{enumerate}
    Then $\frac{\alpha^*}{\beta^*}W^* = D \enspace .$
\end{restatable}

\begin{proof}[Sketch of proof (full proof in \autoref{sec:proof:thm_coupling})]

Let $\varepsilon > 0~,$ and $x \in \mathcal{B}_{\infty}~.$
Using the triangular inequality, we have
    \begin{eqnarray}
        \|\phi_{\theta^*}(z^*, x) - z^*\| & \le  &
            \|\phi_{\theta^*}(z^*, x) - \phi_{\theta^{(t)}}(z^{(t)}, x)\| +\|\phi_{\theta^{(t)}}(z^{(t)}, x) - z^*\|
    \end{eqnarray}

Since the $z^{(t)}$ and $\theta^{(t)}$ converge, they are valued over a compact set $K$.
The function $f: (z, x, \theta) \mapsto \phi_{\theta}(z, x)$ is continuous, piecewise-linear. It is therefore Lipschitz on $K$.
Hence, we have $\|\phi_{\theta^*}(z^*, x) - \phi_{\theta^{(t)}}(z^{(t)}, x)\| \leq \varepsilon$ for $t$ large enough.
Since $\phi_{\theta^{(t)}}(z^{(t)}, x) = z^{(t+1)}$ and $z^{(t)} \to z^*~,$ $\|\phi_{\theta^{(t)}}(z^{(t)}, x) - z^*\| \leq \varepsilon$ for $t$ large enough.
Finally, $\phi_{\theta^*}(z^*, x) = z^*~.$ \autoref{lemma:wcoupling} allows to conclude.
\end{proof}

\begin{figure}[hbtp]
\centering
\begin{minipage}{.35\textwidth}
    \includegraphics[width=\textwidth]{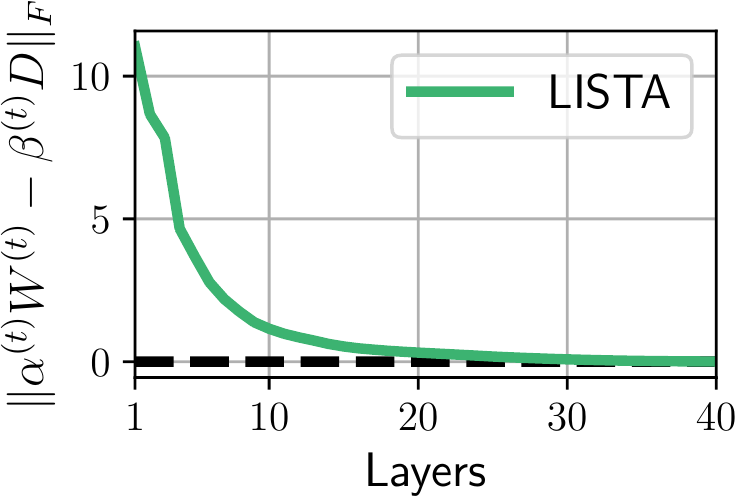}
\end{minipage}%
\begin{minipage}{0.03\textwidth}
\end{minipage}
\hfill
\begin{minipage}{.55\textwidth}
    \caption{Illustration of \autoref{thm:coupling}: for deep layers of LISTA, we have $\|\alpha^{(t)}W^{(t)} - \beta^{(t)}D\|_F \to 0~,$ indicating that the network ultimately only learns a step size.
    Full experimental setup described in \autoref{sec:supp_expe}.}
    \label{fig:fro_smilarity}
\end{minipage}
\begin{minipage}{.03\textwidth}
    \
\end{minipage}
\end{figure}

\Cref{thm:coupling} means that the deep layers of any LISTA network that converges to solutions of the Lasso correspond to SLISTA iterations: $W^{(t)}$ aligns with $D~,$ and $\alpha^{(t)}, \beta^{(t)}$ get coupled.
This is illustrated in \autoref{fig:fro_smilarity}, where a 40-layers LISTA network is trained on a $10 \times 20$ problem with $\lambda = 0.1~.$
As predicted by the theorem, $\frac{\alpha^{(t)}}{\beta^{(t)}}W^{(t)} \to D~.$
The last layers only learn a step size.
This is consistent with the observation of~\citet{Moreau2017} which shows that the deep layers of LISTA stay close to ISTA.
Further, \autoref{thm:coupling} also shows that it is hopeless to optimize the unsupervised objective~\eqref{eq:unsupervised} with $W_{\text{ALISTA}}$~\eqref{eq:alista}, since this matrix is not aligned with $D~.$

\section{Numerical Experiments}
\label{sec:expes}

This section provides numerical arguments to compare SLISTA to LISTA and ISTA.
All the experiments were run using \texttt{Python} \citep{python36} and \texttt{pytorch} \citep{paszke2017automatic}.
The code to reproduce the figures is available online\footnote{
The code can be found in supplementary materials.
}.

\paragraph{Network comparisons}

We compare the proposed approach SLISTA to state-of-the-art learned methods LISTA \citep{Chen2018} and ALISTA \citep{Liu2019} on synthetic and semi-real cases.

In the synthetic case, a dictionary $D \in \Rset^{n\times m}$ of Gaussian i.i.d. entries is generated.
Each column is then normalized to one.
A set of Gaussian i.i.d. samples $(\tilde{x}^i)_{i=1}^N \in \Rset^n$ is drawn.
The input samples are obtained as $x^i = \tilde{x}^i / \|D^{\top}\tilde{x}^i\|_{\infty} \in \mathcal{B}_{\infty}~,$ so that for all $i~,$ $x^i \in \mathcal B_\infty~.$ We set $m=256$ and $n=64$.

For the semi-real case, we used the digits dataset from \texttt{scikit-learn} \citep{scikit-learn} which consists of $8\times 8$ images of handwritten digits from $0$ to $9~.$ We sample $m = 256$ samples at random from this dataset and normalize it do generate our dictionary $D~.$ Compared to the simulated Gaussian dictionary, this dictionary has a much richer correlation structure, which is known to imper the performances of learned algorithms \citep{Moreau2017}.
The input distribution is generated as in the simulated case.

The networks are trained by minimizing the empirical loss $\mathcal L$ \eqref{eq:unsupervised} on a training set of size $N_{\text{train}} = 10,000$ and we report the loss on a test set of size $N_{\text{test}} = 10,000~.$ Further details on training are in \autoref{sec:supp_expe}.

\begin{figure}[hbtp!]
    \centering
    \includegraphics[width=\textwidth]{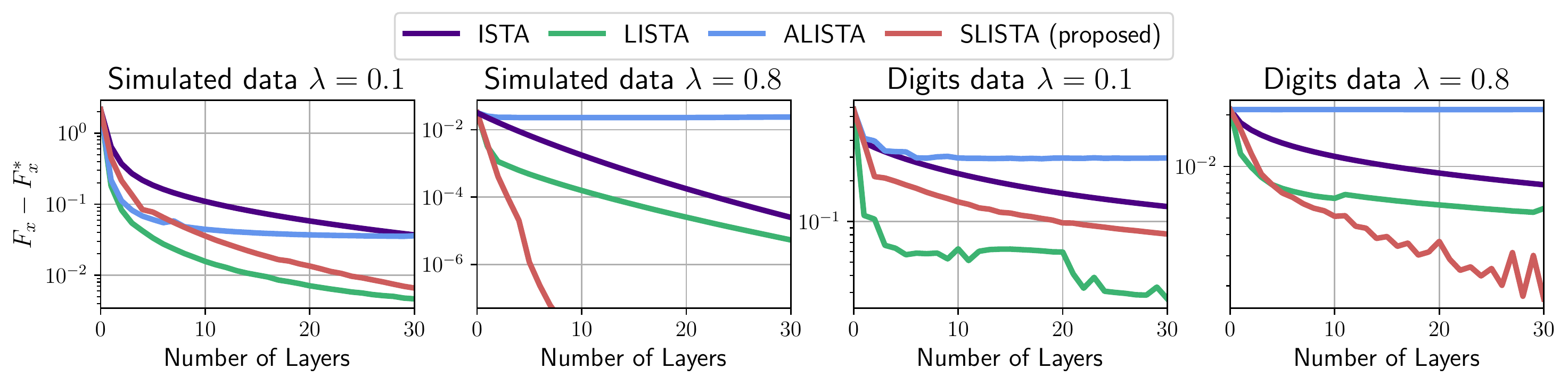}
    \caption{Test loss of ISTA, ALISTA, LISTA and SLISTA on simulated and semi-real data for different regularization parameters.}
    \label{fig:comparison_networks}
\end{figure}

\autoref{fig:comparison_networks} shows the test curves for different levels of regularization $\lambda = 0.1$ and $0.8$.
SLISTA performs best for high $\lambda$, even for challenging semi-real dictionary $D~.$
In a low regularization setting, LISTA performs best as SLISTA cannot learn larger steps due to the low sparsity of the solution.
In this unsupervised setting, ALISTA does not converge in accordance with \autoref{thm:coupling}.

\section{Conclusion}

We showed that using larger step sizes is an efficient strategy to accelerate ISTA for sparse solution of the Lasso.
In order to make this approach practical, we proposed SLISTA, a neural network architecture which learns such step sizes. %
\autoref{thm:coupling} shows that the deepest layers of any converging LISTA architecture must converge to a SLISTA layer.
Numerical experiments show that SLISTA outperforms LISTA in a high sparsity setting.
An major benefit of our approach is that it preserves the dictionary.
We plan on leveraging this property to apply SLISTA in convolutional or wavelet cases, where the structure of the dictionary allows for fast multiplications.

\bibliographystyle{plainnat}
\bibliography{biblio}

\newpage
\appendix
\setcounter{figure}{0}
\renewcommand\thefigure{\thesection.\arabic{figure}}
\renewcommand\thetable{\thesection.\arabic{table}}

\section{Unfolded optimization algorithms literature summary}

In \autoref{tab:literature_network}, we summarize the prolific literature on learned unfolded optimization procedures for sparse recovery. A particular focus is set on the chosen training loss training which is either supervised, with a regression of $z^i$ from the input $x^i$ for a given training set $(x^i, z^i)$, or unsupervised, where the objective is to minimize the Lasso cost function $F_x$ for each training point $x$.

\begin{table}[hbtp]
    \centering
    \caption{Neural network for sparse coding}
    \label{tab:literature_network}
    \begin{tabular}{|c|c|c|>{\centering}m{4em}|>{\centering}m{8em} | H@{\hspace*{-\tabcolsep}}}
        \hline
        Reference              & Base Algo & Train Loss   & Coupled weights  & Remarks     & \\\hline
        \citet{Gregor2010}     & ISTA / CD & supervised   & \xmark &      --               & \\\hline
        \citet{Sprechmann2012} & Block CD  & unsupervised & \xmark & Group $\ell_1$        & \\\hline
        \citet{Sprechmann2013a}& ADMM      & supervised   & N/A    &      --               & \\\hline
        \citet{Hershey2014}    & NMF       & supervised   & \xmark & NMF                   & \\\hline
        \citet{Wang2015}       & IHT       & supervised   & \xmark & Hard-thresholding     & \\\hline
        \citet{Xin2016}        & IHT       & supervised   & \xmark/\cmark & Hard-thresholding     & \\\hline
        \citet{Giryes2016}     & PGD/IHT   & supervised   & N/A    & Group $\ell_1$        & \\\hline
        \citet{Yang2017}       & ADMM      & supervised   & N/A    &      --               & \\\hline
        \citet{Adler2017}      & ADMM      & supervised   & N/A    & Wasserstein distance with $z^*$ & \\\hline
        \citet{Borgerding2017} & AMP       & supervised   & \xmark &    --                 & \\\hline
        \citet{Moreau2017}     & ISTA      & unsupervised & \xmark &    --                 & \\\hline
        \citet{Chen2018}       & ISTA      & supervised   & \cmark & Linear convergence rate & \\\hline
        \citet{Ito2018}        & ISTA      & supervised   & \cmark & MMSE shrinkage non-linearity  & \\\hline
        \citet{Zhang2018}      & PGD       & supervised   & \cmark & Sparsity of Wavelet coefficients& \\\hline
        \citet{Liu2019}        & ISTA      & supervised   & \cmark & Analytic weight $W_{\text{ALISTA}}$   & \\\hline
        \textbf{Proposed}      & ISTA      & unsupervised & \cmark &    --                 & \\\hline
    \end{tabular}
\end{table}

\section{Proofs of \autoref{sec:steps}'s results}
\label{sec:proof3}

\subsection{Proof of \autoref{prop:cvgfinite}}
\label{sub:proof:cvgfinite}

We consider that the solution of the Lasso is unique, following the result of \citet{Tibshirani13}[Lemmas 4 and 16] when the entries of $D$ and $x$ come from a continuous distribution.

\cvgfinite*

\begin{proof}
    Let $z^{(t)}$ be the sequence of iterates produced by \Cref{algo:oista}.
    We have a \emph{descent function}
    \begin{equation}
        F_x(z^{(t+1)}) - F_x(z^{(t)}) \leq - \frac \gamma 2 \Vert z^{(t + 1)} - z^{(t)} \Vert^2  \leq - \frac{\min \Vert D_j \Vert}{2} \Vert z^{(t + 1)} - z^{(t)} \Vert^2 \enspace,
    \end{equation}
    where $\gamma = L_S$ if \autoref{cond:star} is met, and $L$ otherwise.
    Additionally, the iterates are bounded because $F_x(z^{(t)})$ decreases at each iteration and $F_x$ is coercive.
    Hence we can apply Zangwill’s Global Convergence Theorem \citep{Zangwill69}.
    Any $z^*$ accumulation point of $(z^{(t)})_{t \in \Nset}$ is a minimizer of $F_x~.$

    Since we only consider the case where the minimizer is unique, the bounded sequence $(z^{(t)})_{t \in \Nset}$ has a unique accumulation point, thus converges to $z^*~.$

    The support identification is a simplification of a result of \citet{Hale_Yin_Zhang2008}, we include it here for completeness.

    \begin{lemma}[Approximation of the soft-thresholding]
    Let $z \in \Rset, \nu > 0~.$ For $\epsilon$ small enough,
    we have
    \begin{equation}
        \st(z + \epsilon, \nu) =
        \begin{cases}\begin{aligned}
        &0 \enspace, &\mathrm{\quad if \,} |z| < \nu \enspace, \\
        &\max(0, \epsilon) \sign(z) \enspace, &\mathrm{\quad if \,} |z| = \nu \enspace,\\
        &z + \epsilon - \nu \sign{z} \enspace, &\mathrm{\quad if \,} |z| > \nu \enspace.
       \end{aligned} \end{cases} \label{eq:st_linearization}
    \end{equation}
    \end{lemma}

    Let $\rho > 0$ be such that \Cref{eq:st_linearization} holds for $\nu = \lambda / L~,$ every $\epsilon < \rho~,$ and every $z = z_j^* - \frac1L D_j^\top(Dz^* - x)~.$

    Let $t \in \Nset$ such that $z^{(t)} = z^* + \epsilon~,$ with $\Vert \epsilon \Vert \leq \rho~.$ With $\epsilon' \triangleq (\Id - \frac 1L D^\top D) \epsilon~,$ we also have $\Vert \epsilon' \Vert \leq \rho~.$
    Let $j \in \llbracket 1, m \rrbracket~.$

    If $j \notin E~,$ $\vert z_j^* - \frac1L D_j^\top(Dz^* - x) \vert = \vert \frac1L D_j^\top(Dz^* - x) \vert < \lambda / L$ hence $\st( z_j^* - \frac1L D_j^\top(Dz^* - x) +  \epsilon'_j, \lambda/L) = 0~.$

    If $j \in E~,$ $\vert z_j^* - \frac1L D_j^\top(Dz^* - x) \vert = | z^*_j +\frac{\lambda}{L} \sign z^*_j | > \lambda / L~,$ and
    $\sign \st( z_j^* - \frac1L D_j^\top(Dz^* - x) +  \epsilon'_j, \lambda/L) = \sign z^*_j~.$

    The same reasoning can be applied  with $\rho'$ such that  \Cref{eq:st_linearization} holds for $\nu = \lambda / L_{S^*}~,$ every $\epsilon < \rho'~,$ and every $z = z_j^* - \frac1{L_S^{*}} D_j^\top(Dz^* - x)$.
    If we introduce $\eta > 0$ such that $\Vert \epsilon \Vert \leq \eta \implies \Vert (\Id - \frac{1}{L_{S^*}} D\top D) \epsilon \Vert \leq \rho' ~,$
    in the ball of center $z^*$ and radius $\eta~,$ the iteration with step size $L_{S^*}$ identifies the support.

    Additionnally, since $\Id - \frac{1}{L_{S^*}} D_{S^*}^\top D_{S^*}$ is non expansive on vectors which support is $S^*~,$ the iterations with the step $L_{S^*}$ never leave this ball once they have entered it.

    Therefore, once the iterates enter $\mathcal{B}(z^*, \min(\eta, \rho))~,$ \autoref{cond:star} is always satisfied.

\end{proof}

\subsection{Proof of \autoref{prop:rates}}
\label{sub:proof:rates}

\sublinear*
\begin{proof}
    For $t \geq T^*~,$ the iterates support is $S^*$ and the objective function is $L_{S^*}$-smooth instead of $L$-smooth.
    It is also $\mu^*$ strongly convex if $\mu^*>0~.$
    The obtained rates are a classical result of the proximal gradient descent method in these cases.
\end{proof}

\section{Proof of \autoref{sec:learning}'s Lemmas}
\label{sec:proof:lemma}

\subsection{Proof of \autoref{lemma:stability}}

\label{sec:proof:stability}

\stability*
\begin{proof}
    Let $j \in \llbracket 1, m \rrbracket$ and let $\varepsilon \in \Rset^{m} \cap D_j^\perp$ be a vector such that $\|\varepsilon\| < \lambda(1-c)~.$\\
    For notation simplicity, we denote $z^* = z^*(D_j - \varepsilon)~.$
    \begin{align}
        D^\top_j(Dz^* - D_j - \varepsilon)
        &= D^\top_j(- \lambda D_j - \varepsilon) = - \lambda = - \lambda \sign z^*_j ~,
    \end{align}
    since $1 - \lambda > 0~.$
    For the other coefficients $l \in\llbracket 1, m \rrbracket \setminus \{j\}~,$ we have
    \begin{align}
    |D^\top_l(Dz^* - D_j - \varepsilon)|
      &= |D^\top_l(- \lambda D_j - \varepsilon)|~,\\
      &= | \lambda D^\top_lD_j
                + D^\top_l\varepsilon)|~,\\
      &\le \lambda |D^\top_lD_j| + |D_l^\top\varepsilon|~,\\
      &\le \lambda c + \|D_l\|\|\varepsilon\|~,\\
      &\le \lambda c + \|\varepsilon\| < \lambda~,\\
    \end{align}
    Therefore, $(1  - \lambda) e_j$ verifies the KKT conditions \eqref{eq:KKT} and $z^*(D_j + \varepsilon) = (1  - \lambda) e_j~.$
\end{proof}

\subsection{Proof of \autoref{lemma:wcoupling}}
\label{sec:proof:wcoupling}

\wcoupling*

\begin{proof}
    Let $x\in\mathcal B_\infty$ be an input vector and $z^*(x)\in\Rset^m$ be a solution for the Lasso at level $\lambda > 0~.$
    Let $j \in \llbracket 1, m \rrbracket$ be such that $z^*_j > 0~.$
    The KKT conditions \eqref{eq:KKT} gives
    \begin{equation}
        \label{eq:weigh_coupling:kkt}
        D_j^\top(Dz^*(x) - x) = -\lambda \enspace.
    \end{equation}
     Suppose that $z^*(x)$ is a fixed point of the layer, then we have
     \begin{equation}
         \st(z^*_j(x) - \alpha W_j^\top(Dz^*(x) - x), \lambda\beta) = z_j^*(x) > 0 \enspace.
     \end{equation}
     By definition, $\st(a, b) > 0$ implies that $a > b$ and $\st(a, b) = a - b~.$ Thus,
     \begin{align}
         &z^*_j(x) - \alpha W_j^\top(Dz^*(x) - x) - \lambda\beta = z_j^*(x)\\
         \Leftrightarrow ~~~~~& \alpha W_j^\top(Dz^*(x) - x) + \lambda\beta = 0
         \\
         \Leftrightarrow ~~~~~& \alpha W_j^\top(Dz^*(x) - x) - \beta D_j^\top(Dz^*(x) - x) = 0 \quad \quad \text{by \eqref{eq:weigh_coupling:kkt}}
         \\
         \Leftrightarrow ~~~~~& (\alpha W_j - \beta D_j)^\top(Dz^*(x) - x) = 0 \enspace .
         \label{eq:weight_coupling:relation}
     \end{align}
     As the relation \eqref{eq:weight_coupling:relation} must hold for all $x\in\mathcal B_\infty~,$ it is true for all $D_j + \varepsilon$ for all $\varepsilon \in \mathcal B(0, \lambda (1-c))\cap D_j^\perp~.$ Indeed, in this case, $\|D^\top (D_j + \varepsilon)\|_\infty = 1~.$
     $D$ verifies the conditions of \Cref{lemma:stability}, and thus $z^* = (1-\lambda)e_j~,$ \ie{}
      \begin{align}
          (\alpha W_j - \beta D_j)^\top (D(1  - \lambda) e_j - (D_j + \varepsilon)) &= 0 \\
          (\alpha W_j - \beta D_j)^\top \left(- \lambda D_j - \varepsilon \right) &= 0
          \label{eq:stability}
     \end{align}
     Taking $\varepsilon = 0$ yields $(\alpha W_j - \beta D_j)^\top D_j = 0~,$ and therefore Eq.~\eqref{eq:stability} becomes $(\alpha W_j - \beta D_j)^\top \varepsilon = 0$ for all $\varepsilon$ small enough and orthogonal to $D_j~,$ which implies $\alpha W_j - \beta D_j = 0$ and concludes our proof.
\end{proof}

\subsection{Proof of \autoref{thm:coupling}}
\label{sec:proof:thm_coupling}

\coupling*

\begin{proof}
     For simplicity of the notation, we will drop the $x$ variable whenever possible, \ie{} $z^* = z^*(x)$ and $\phi_\theta(z) = \phi_\theta(z, x)~.$ We denote $z^{(t)} = \Phi_{\Theta^{(t)}}(x)$ the output of the network with $t$ layers.\\
    Let $\epsilon > 0~.$ By hypothesis \ref{hyp:theta_cvg}, there exists $T_0$ such that for all $t \ge T_0~,$
    \begin{equation}
        \label{eq:thm:parameter_cvg}
        \|W^{(t)} - W^*\| \le \epsilon ~~~
        |\alpha^{(t)} - \alpha^*| \le \epsilon ~~~
        |\beta^{(t)} - \beta^*| \le \epsilon~.
    \end{equation}
    By hypothesis \ref{hyp:out_cvg}, , there exists $T_1$ such that for all $t \ge T_1$ and all $x \in \mathcal B_\infty~,$
    \begin{equation}
        \label{eq:thm:iterate_cvg}
       \|z^{(t)} - z^*\| \le \epsilon ~.
    \end{equation}
    Let $x\in\mathcal B_\infty$ be an input vector and $t \ge \max(T_0, T_1)~.$
    Using \eqref{eq:thm:iterate_cvg}, we have
    \begin{eqnarray}
        \|z^{(t+1)} - z^{(t)}\| & \le &  \|z^{(t+1)} - z^*\| +  \|z^{(t)} - z^*\| \le 2\epsilon
        \label{eq:thm:cauchy_cvg}
    \end{eqnarray}
    By \ref{hyp:theta_cvg}, there exist a compact set $\mathcal K_1 \subset \Rset^{n \times m} \times \Rset^+_* \times \Rset^+_*$ \emph{s.t.} $\theta^{(t)} \in \mathcal K_1$ for all $t \in\mathbb N$ and $\theta^* \in \mathcal K~.$ The input $x$ is taken in a compact set $\mathcal B_\infty$ and as $z^* = \argmin_z F_x(z)~,$ we have $\lambda\|z\|_1 \le F_x(z^*) \le F_x(0) = \|x\|$ thus $z^*$ is also in a compact set $\mathcal K_2~.$\\
    We consider the function $f(z, x, \theta) = \st(z - \alpha W^\top(Dz - x), \beta)$ on the compact set $\mathcal K_2 \times \mathcal B_\infty \times \mathcal K_1~.$ This function is continuous and piece-wise linear on a compact set. It is thus $L$-Lipschitz and thus
    \begin{eqnarray}
        \|\phi_{\theta^{(t)}}(z^{(t)}) - \phi_{\theta^{(t)}}(z^*)\|
            & \le & L\|z^{(t)} - z^*\| \le L\epsilon \label{eq:thm:shs}\\
        \|\phi_{\theta^*}(z^*) - \phi_{\theta^{(t)}}(z^*)\|
            & \le & L\|\theta^{(t)} - \theta^*\| \le L\epsilon \label{eq:thm:fhs}
    \end{eqnarray}
    Using these inequalities, we get
    \begin{eqnarray}
        \|\phi_{\theta^*}(z^*, x) - z^*\|
            & \le  &
            \underbrace{\|\phi_{\theta^*}(z^*) - \phi_{\theta^{(t)}}(z^*)\|}_{< L\epsilon~\text{by \eqref{eq:thm:fhs}}}
            +\underbrace{\|\phi_{\theta^{(t)}}(z^*) - \phi_{\theta^{(t)}}(z^{(t)})\|}_{< L\epsilon~\text{by \eqref{eq:thm:shs}}}\\
            && +\underbrace{\|\phi_{\theta^{(t)}}(z^{(t)}) - z^{(t)}\|}_{< 2\epsilon ~\text{ by \eqref{eq:thm:cauchy_cvg}}}
            +\underbrace{\|z^{(t)} - z^*\|}_{< \epsilon ~\text{ by \eqref{eq:thm:iterate_cvg}}} \nonumber\\
            & \le & (2L + 3)\epsilon~.
    \end{eqnarray}
    As this result holds for all $\epsilon >0$ and all $x\in\mathcal B_\infty~,$ we have $\phi_{\theta^*}(z^*) = z^*$ for all $x \in \mathcal B_\infty~.$ We can apply the \autoref{lemma:wcoupling} to conclude this proof.
\end{proof}

\section{Experimental setups and supplementary figures}
\label{sec:supp_expe}

\textbf{Dictionary generation}: Unless specified otherwise, to generate synthetic dictionaries, we first draw a random i.i.d. Gaussian matrix $\hat{D} \in \Rset^{n \times m}$.
The dictionary is obtained by normalizing the columns:
$D_{ij} = \frac{1}{\|\hat{D}_{i:}\|}\hat{D}_{ij}$.

\textbf{Samples generation}: The samples $x$ are generated as follows: Random i.i.d. Gaussian samples $\hat{x}\in \Rset^n$ are generated. We then normalize them: $x = \frac{1}{\|D^{\top}\hat{x}\|_{\infty}}\hat{x}$, so that $x \in \mathcal B _{\infty}$.

\textbf{Training the networks} Since the loss function and the network are continuous but non-differentiable, we use sub-gradient descent for training.
The sub-gradient of the cost function with respect to the parameters of the network is computed by automatic differentiation.
We use full-batch sub-gradient descent with a backtracking procedure to find a suitable learning rate.
To verify that we do not overfit the training set, we always check that the test loss and train loss are comparable.

\paragraph{Main text figures setup}
\begin{itemize}
    \item \autoref{fig:oista_ista}: We generate a random dictionary of size $10 \times 50$. We take $\lambda=0.5$, and a random sample $x\in \mathcal{B}_{\infty}$. $F_x^*$ is computed by iterating ISTA for $10000$ iterations.
    \item \autoref{fig:learned_steps}: We generate a random dictionary of size $10 \times 20$. We take $\lambda=0.2$.
    We generate a training set of $N=1000$ samples $(x^i)_{i=1}^{1000}\in\mathcal B _{\infty}$.
    A 20 layers SLISTA network is trained by gradient descent on these data. We report the learned step sizes.
    For each layer $t$ of the network and each training sample $x$, we compute the support at the output of the $t$-th layer, $S(x, t)= \supp(z^{(t)}(x))$.
    For each $t$, we display the quantiles of the distribution of the $(1/L_{S(x^i, t)})_{i=1}^{1000}$.
    \item \autoref{fig:fro_smilarity}: A random $10 \times 20$ dictionary is generated.
    We take $1000$ training samples, and $\lambda=0.05$. A $40$ layers LISTA network is trained by gradient descent on those samples.
    We report the quantity $\|\alpha^{(t)}W^{(t)} - \beta^{(t)}D\|_F$ for each layer $t$.
\end{itemize}

\pagebreak
\paragraph{\color{black}Supplementary experiments}

{\color{white}.}

\vskip2em
\begin{figure}[hp!]
    \begin{minipage}{.4\textwidth}
    \centering
    \includegraphics[width=\columnwidth]{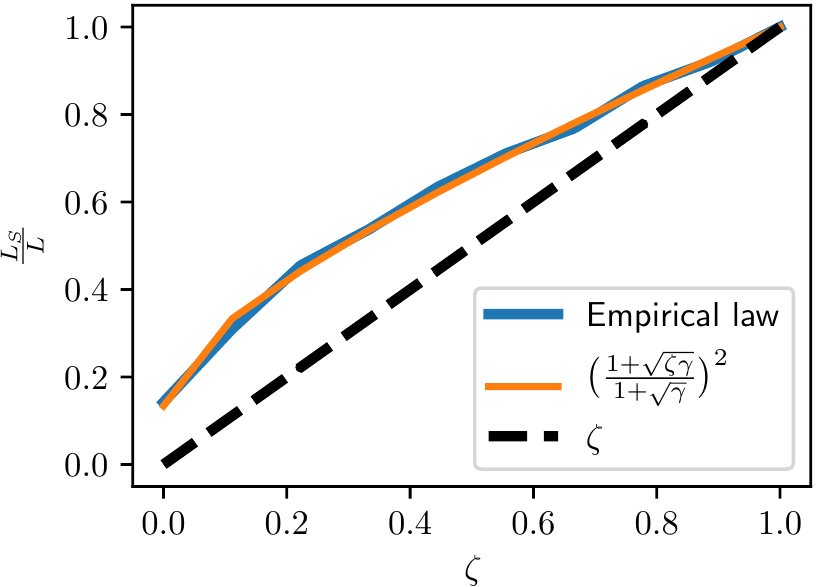}
    \end{minipage}\hfill %
    \begin{minipage}{.58\textwidth}
        \caption{Illustration of~\Cref{prop:eig}. A toy Gaussian dictionary is generated with $n=200~,$ $m=600$ so that $\gamma=3~.$
        We compute its Lipschitz constant $L~.$
        For $\zeta$ between $0$ and $1~,$ we extract $\lfloor \zeta m \rfloor$ columns at random and compute the corresponding Lipschitz constant $L_S~.$ The plot shows an almost perfect fit between the empirical law and the theoretical limit~\eqref{eq:ratioL}.}
        \label{fig:lip_distrib}
    \end{minipage}
\end{figure}

\vskip2em
\begin{figure}[hp!]
    \centering
        \includegraphics[width=.65\textwidth]{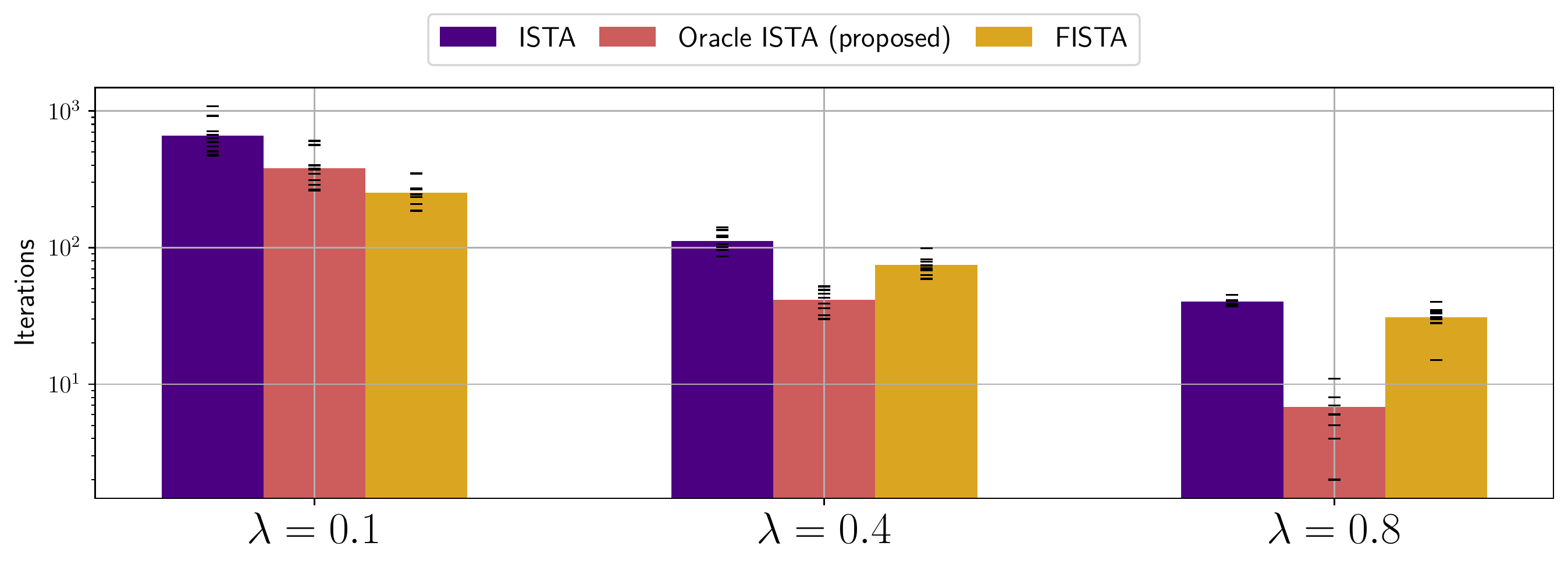}
    \caption{Comparison between ISTA, FISTA and Oracle-ISTA for different levels of regularization on a Gaussian dictionnary, with $n=100$ and $m=200$. We report the average number of iterations taken to reach a point $z$ such that  $F_x(z) < F^*_x + 10^{-13}$. The experiment is repeated $10$ times, starting from random points in $\mathcal{B}_{\infty}$.
    OISTA is always faster than ISTA, and is faster than FISTA for high regularization.}
    \label{fig:comparison_oista}
\end{figure}

\vskip2em
\begin{figure}[hp!]
    \begin{minipage}{.4\textwidth}
        \centering
        \includegraphics[width=\columnwidth]{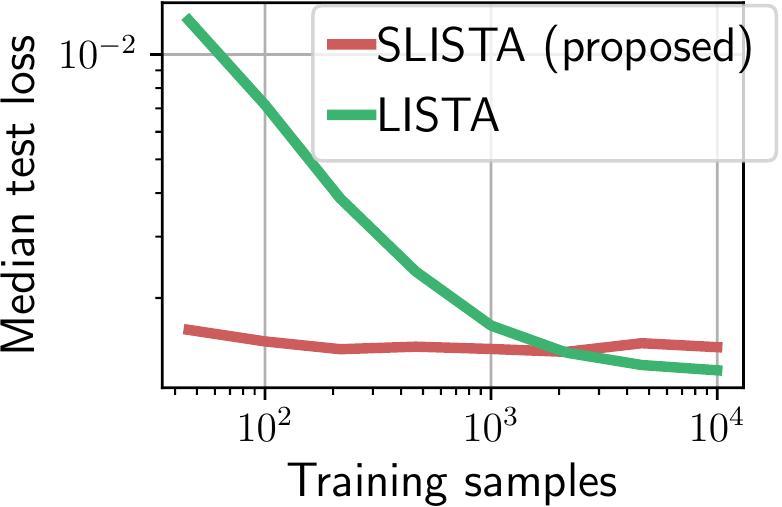}
    \end{minipage}
    \hfill
    \begin{minipage}{.58\textwidth}
        \centering
        \caption{Learning curves of SLISTA and LISTA.
        Random Gaussian dictionaries with $n=10$ and $m=20$ are generated.
        We take $\lambda=0.3$.
        Networks with 10 layers are fit on those dictionaries, and their test loss is reported for different number of training samples.
        The process is repeated $100$ times; the curves shown display the median of the test-loss.
        }
    \end{minipage}
\label{fig:learning_curve}
\end{figure}
\end{document}